\RequirePackage{fix-cm}
\documentclass[smallextended]{svjour3}       
\smartqed  
\usepackage{graphicx}
\usepackage{url}
\usepackage{xcolor}
\usepackage{comment}
\usepackage{tikz}
\usepackage{amssymb}
\usepackage{amsmath}
\usepackage{mathtools}
\DeclarePairedDelimiter{\ceil}{\lceil}{\rceil}
\usepackage{amsfonts} 
\usepackage{pgfplots}
\usepackage{booktabs}  
\usepackage{nicefrac} 
\usepackage{blindtext}
\usepackage[numbers,sort&compress]{natbib}
\usepackage{multirow}
\pgfplotsset{compat=1.15}
\newcommand{\cri}[1]{\textcolor{black}{#1}}

\begin{document}

\title{Voting with Random Classifiers (VORACE): Theoretical and Experimental Analysis
}

\titlerunning{Voting with Random Classifiers (VORACE)}        

\author{Cristina Cornelio  \and Michele Donini \and Andrea Loreggia \and Maria Silvia Pini \and Francesca Rossi \thanks{This is a preprint of an article published in Autonomous Agents and Multi-Agent Systems journal. The final
authenticated version is available online at: https://doi.org/10.1007/s10458-021-09504-y} \thanks{A. Loreggia has been supported by the H2020 ERC Project “CompuLaw” (G.A.833647)}}


\institute{
Cristina Cornelio \at
              IBM Research, R\"uschlikon, Z\"urich, Switzerland \\
              \email{cor@zurich.ibm.com}      
              \and
	Michele Donini\textsuperscript{*} \at
             Amazon, Berlin, Germany  \\
              \email{donini@amazon.com}  
              \and
	Andrea Loreggia  \at
              European University Institute, Firenze, Italy \\
              \email{andrea.loreggia@gmail.com}  
                    \and
	Maria Silvia Pini \at
              Department of Information Engineering, University of Padova, Italy \\
              \email{pini@dei.unipd.it}  
           \and
	Francesca Rossi	\at
              IBM Research, Yorktown Heights, New York, USA \\
              \email{francesca.rossi2@ibm.com }  
              \and
    \textsuperscript{*}This work was mainly conducted prior joining Amazon. }

\date{Received: \today / Accepted: date}

\maketitle

\begin{abstract}
In many machine learning scenarios, looking for the best classifier that fits a particular dataset can be very costly in terms of time and resources. Moreover, it can require deep knowledge of the specific domain.
We propose a new technique which does not require profound expertise in the domain and avoids the commonly used strategy of hyper-parameter tuning and model selection. Our method is an innovative ensemble technique that uses voting rules over a set of randomly-generated classifiers.
Given a new input sample, we interpret the output of each classifier as a ranking over the set of possible classes.
We then aggregate these output rankings using a voting rule, which treats them as preferences over the classes. We show that our approach obtains good results compared to the state-of-the-art, both providing a theoretical analysis and an empirical evaluation 
of the approach on several datasets. 
\end{abstract}

\section{Introduction}

It is not easy to identify the best classifier for a certain complex task \cite{Gul2018EnsembleOA,Perikos2016RecognizingEI,DBLP:journals/paa/BarandelaVS03}. 
Different classifiers may be able to exploit better the features of different regions of the domain at hand, and consequently their accuracy might be better only in that region \cite{DBLP:journals/tsmc/KhoshgoftaarHN11,DBLP:conf/mcs/MelvilleSMM04,DBLP:journals/ml/BauerK99}.
Moreover, fine-tuning the classifier's hyper-parameters is a time-consuming task, which also requires a deep knowledge of the domain and a good expertise in tuning various kinds of classifiers.
Indeed, the main approaches to identify the hyper-parameters' best values are either manual or based on grid search, although there are some approaches based on random search \cite{BergstraB12}. 
However, it has been shown that in many scenarios
there is no single learning algorithm that can uniformly outperform the others over all data sets \cite{KotsiantisZP06,rokach2010ensemble,gandhi2015hybrid}. This observation led to an alternative approach to improve the performance of a classifier, which consists of combining several different classifiers (that is, an {\em ensemble} of them) and taking the class proposed by their combination.
Over the years, many researchers have studied methods for constructing good ensembles of classifiers \cite{kittler:icpr96,rokach2010ensemble,KotsiantisZP06,gandhi2015hybrid,DBLP:journals/apin/NetoC18,Dietterich00}, showing that indeed ensemble classifiers are often much more accurate than the individual classifiers within the ensemble \cite{kittler:icpr96}.  
Classifiers combination is widely applied to many different fields, such as urban environment classification \cite{DBLP:journals/aeog/AzadbakhtFK18,DBLP:journals/ijgi/SunLSH17} and medical decision support \cite{DBLP:journals/artmed/SalehBMVRRS18,DBLP:journals/cee/AteeqMAMRLMWBM18}. 
In many cases, the performance of an ensemble method cannot be easily formalized theoretically, but it can be easily evaluated on an experimental basis in specific working conditions (that is, a specific set of classifiers, training data, etc.).

In this paper we propose a new ensemble classifier method, called VORACE, which aggregates randomly generated classifiers using voting rules in order to provide an accurate prediction for a supervised classification task. Besides the good accuracy of the overall classifier, one of the main advantages of using VORACE is that it does not require specific knowledge of the domain or good expertise in fine-tuning the classifiers' parameters.  

We interpret each classifier as a voter, whose vote is its prediction over the classes, and a voting rule aggregates such votes to identify the "winning" class, that is, the overall prediction of the ensemble classifier. 
This use of voting rules is within the framework of maximum likelihood estimators, where each vote (that is, a classifier's rank of all classes) is interpreted as a noisy perturbation of the correct ranking (that is not available), so a voting rule is a way to estimate this correct ranking  \cite{handbooksocialchoice,ConitzerRX09,ConitzerMLE}. 
To the best of our knowledge, this is the first attempt to combine randomly generated classifiers, to be aggregated in an ensemble method, using voting theory to solve a supervised learning task without exploiting any knowledge of the domain. We theoretically and experimentally show that the usage of generic classifiers in an ensemble environment can give results that are comparable with other state-of-the-art ensemble methods. 
Moreover, we provide a closed formula to compute the performance of our ensemble method in the case of Plurality, this corresponds to the probability of choosing the correct class, assuming that all the classifiers are independent and have the same accuracy. We then relax these assumptions by defining the probability of choosing the right class when the classifiers have different accuracies and they are not independent.

Properties of many voting rules have been studied extensively in the literature \cite{handbooksocialchoice,grandi2016borda}. So another advantage of using voting rules is that we can exploit that literature to make sure certain desirable properties of the resulting ensemble classifier hold. Besides the classical properties that the voting theory community has considered (like anonymity, monotonicity, IIA, etc.), there may be also other properties not yet considered, such as various forms of fairness \cite{LoMaRoVe18,rossi2019preferences}, whose study is facilitated by the use of voting rules. 

The paper is organized as follows.
In Section \ref{sec-rw} we briefly describe some prerequisites (a brief introduction to ensemble methods and voting rules) necessary for what follows and an overview of previous works in this research area.
In Section \ref{sec-vorace} we present our approach that exploits voting theory in the ensemble classifier domain using neural networks, decision trees, and support vector machines. 
In Section \ref{sec-exp} we show our experimental results, while in Sections \ref{sec:tech-analysis}, \ref{sec:Condorcet} and \ref{sec:relaxing_assumptions} we discuss our theoretical analysis: in Section \ref{sec:tech-analysis} we present the case in which all the classifiers are independent and with the same accuracy; in Section \ref{sec:Condorcet} we relate our results with the Condorcet Jury Theorem also showing some interesting properties of our formulation (e.g. monotonicity and behaviour with infinite voters/classifiers); and in Section \ref{sec:relaxing_assumptions} we extend the results provided in Section \ref{sec:tech-analysis} relaxing the assumptions of having all the classifiers with the same-accuracy and independent between each other.
Finally, in Section \ref{sec-con} we summarize the results of the paper and we give some hints for future work.

A preliminary version of this work has been published as an extended abstract at the  International Conference On Autonomous Agents and Multi-Agent Systems (AAMAS-20) \cite{VORACE_AAMAS20}.  The code is available open source at \url{https://github.com/aloreggia/vorace/}.

\section{Background and Related Work}\label{sec-rw}
\subsection{Ensemble methods}
Ensemble methods combine multiple classifiers in order to give a substantial improvement in the prediction performance of learning algorithms, especially for datasets which present non-informative features \cite{DBLP:journals/adac/GulPKMMAL18}. 
Simple combinations have been studied from a theoretical point of view, and many different ensemble methods have been proposed \cite{kittler:icpr96}.
Besides simple standard ensemble methods (such as averaging, blending, staking, etc.), {\it Bagging} and {\it Boosting} can be considered two of the main state-of-the-art ensemble techniques in the literature \cite{rokach2010ensemble}. In particular, 
{\bf Bagging} \cite{DBLP:journals/ml/Breiman96b}
trains the same learning algorithm on different subsets of the original training set. These different training subsets are generated by randomly drawing, with replacement, $N$ instances, where $N$ is the original size of training set. Original instances may be repeated or left out. This allows for the construction of several different classifiers where each classifier can have specific knowledge of part of the training set. Aggregating the predictions of the individual classifiers leads to the final overall prediction. Instead, 
{\bf Boosting} \cite{DBLP:journals/jcss/FreundS97} 
keeps track of the learning algorithm performance in order to focus the training attention on instances that have not been correctly learned yet. Instead of choosing training instances at random from a uniform distribution, it chooses them in a manner as to favor the instances for which the classifiers are predicting a wrong class. The final overall prediction is a weighted vote (proportional to the classifiers' training accuracy) of the predictions of the individual classifiers.

While the above are the two main approaches, other variants 
have been proposed, such as Wagging \cite{DBLP:journals/ml/Webb00}, MultiBoosting \cite{DBLP:journals/ml/Webb00}, and Output Coding \cite{DBLP:journals/jair/DietterichB95}.
We compare our work with the state-of-the-art in ensemble classifiers, in particular XGBoost \cite{chen2016xgboost}, which is based on boosting, and  Random Forest (RF) \cite{ho1995random}, which is based on bagging.



\subsection{Voting rules}\label{sec-voting}

For the purpose of this paper, a {\it voting rule} is a procedure that allows a set of voters to collectively choose one among a set of candidates.
Voters submit their vote, that is, their preference ordering over the set of candidates, and the voting rule aggregates such votes to yield a final result (the winner).
In our ensemble classification scenario, the voters are the individual classifiers and the candidates are the classes. A vote is a ranking of all the classes, provided by an individual classifier.
In the classical voting setting, 
given a set of $n$ voters (or agents) $A = \{a_1 , . . . , a_n \}$ and $m$ candidates $C=\{c_1, . . . , c_m\}$, a {\it profile} is a collection of $n$ total orders over the set of candidates, one for each voter. 
So, formally, a voting rule 
is a map from a profile to a winning candidate\footnote{
We assume that there is always a unique winning candidate. In case of ties between candidates, we will use a predefined tie-breaking rule to choose one of them to be the winner.}.
The voting theory literature includes 
many voting rules, with different properties. In this paper, we focus on four of them, but the approach is applicable also to any other voting rules:

\noindent {\bf 1) Plurality:} Each voter states who the preferred candidate is, without providing information about the other less preferred candidates. The winner is the candidate who is preferred by the largest number of voters.

\noindent {\bf 2) Borda:} Given $m$ candidates, each voter gives a ranking
of all candidates. Each candidate receives a score for each voter, based on its position in the ranking: the $i$-th ranked candidate gets the score $m - i$. The candidate with the largest sum of all scores wins.

\noindent {\bf 3) Copeland:} Pairs of candidates are compared in terms of how many voters prefer one or the other one, and the winner of such a pairwise comparison is the one with the largest number of preferences over the other one. The overall winner is the candidate who wins the most pairwise competitions against all the other candidates.


\noindent{\bf 4) Kemeny \cite{kemeny1959mathematics}}:
We borrow a formal definition of the rule from \citet{conitzer2006improved}. For any two candidates $a$ and $b$, given a ranking $r$ and a vote $v$, let $\delta_{a,b}(r,v)=1$ if $r$ and $v$ agree on the relative ranking of $a$ and $b$ (e.g., they either both rank $a$ higher, or both rank $b$ higher), and $0$ if they disagree. Let the agreement of a ranking $r$ with a vote $v$ be given by $\sum_{a,b} \delta_{a,b}(r, v)$, the total number
of pairwise agreements. A Kemeny ranking $r$ maximizes the
sum of the agreements with the votes $\sum_{v}\sum_{a,b} \delta_{a,b}(r, v)$.
This is called a Kemeny consensus. A candidate is a winner of a Kemeny election if it is the top candidate in the Kemeny consensus for that election.

It is easy to see that all the above voting rules associate a score to each candidate (although different voting rules associate different scores), and the candidate with the highest score is declared the winner. Ties can happen when more than one candidate results with the highest score, we arbitrarily break the tie lexicographically in the experiments. We plan to test the model on different and more fair tie-breaking rules.
It is important to notice that when the number of candidates is $m=2$ (that is, we have a binary classification task) all the voting rules have the same outcome since they all collapse to the 
{\bf Majority} rule, which elects the candidate which has a majority, that is, more than half the votes. 

Each of these rules has its advantages and drawbacks. Voting theory provides an axiomatic characterization of voting rules in terms of desirable properties such as anonymity, neutrality, etc. -- for more details on voting rules see \cite{DBLP:series/synthesis/2011Rossi,voting-handbook,handbooksocialchoice}.
In this paper, we do not exploit these properties to choose the "best" voting rule, but rather we rely on what the experimental evaluation tells us about the accuracy of the ensemble classifier. 

\subsection{Voting for ensemble methods}\label{sec:voting_ensemble}

Preliminary techniques from voting theory have already been used to combine individual classifiers in order to improve the performance of some ensemble classifier methods \cite{KotsiantisP05,gandhi2015hybrid,DBLP:journals/ml/BauerK99,donini2018voting}.
Our approach differs from these
methods in the way classifiers are generated and how the outputs of the individual classifiers are aggregated. 
Although in this paper we report results only 
against recent bagging and boosting techniques of ensemble classifiers, we compared our approach with the other existing approaches as well.
More advanced work has been done to study the use of a specific voting rule: the use of {\it majority} to ensemble a profile of classifiers has been investigated in the work of \citet{LamSuen}, where they theoretically analyzed the performance of majority voting (with {\it rejection} if the $50\%$ of consensus is not reached) when the classifiers are assumed independent. 
In the work of \citet{Kuncheva2003}, they provide upper and lower limits on the majority vote accuracy focusing on dependent classifiers.
We perform a similar analysis of the dependence between classifier but in the more complex case of plurality, with also an overview of the general case.
Although majority seems to be easier to evaluate compared to plurality, there have been some attempts to study plurality as well:
\citet{LinYacoub} demonstrated some interesting theoretical results for independent classifiers, and 
\citet{Mu2009} extended their work providing a theoretical analysis of the probability of electing the correct class by an ensemble using plurality, or plurality with rejection, as well as a stochastic analysis of the formula, and evaluating it on a dataset for human recognition. 
However, we have noted an issue with their proof: the authors assume independence between the random variable expressing the total number of votes received by the correct class and the one defining the maximum number of votes among all the wrong classes. This false assumption leads to a wrong final formula (the proof can be found in Appendix~\ref{appendix:Mu2009}). In our work, we provide a formula that exploits generating functions and that fixes the problem of \citet{Mu2009}, based on a different approach. Moreover, we provide proof for the two general cases in which the accuracy of the individual classifiers is not homogeneous, and where classifiers are not independent. Furthermore, our experimental analysis is more comprehensive: not limiting to plurality and considering many datasets of different types.
There are also some approaches that use {\it Borda count} for ensemble methods (see for example the work of \citet{Erp00variantsof}).
Moreover, voting rules have been applied to the specific case of Bagging \cite{Baggingvoting1,Baggingvoting2}. However, in \citet{Baggingvoting1}, the authors combine only classifiers from the same family (i.e., 
Naive Bayes classifier) without mixing them.

A different perspective comes from the work of \citet{de2014essai} and further improvements \cite{youngcondorcet,ConitzerRX09,ConitzerMLE} where the basic assumption is that there always exists a correct ranking of the alternatives, but this cannot be observed directly. Voters derive their preferences over the alternatives from this ranking (perturbing it with noise). Scoring voting rules are proved to be maximum likelihood estimators (MLE). Under this approach, one computes the likelihood of the given preference profile for each possible state of the world, that is, the true ranking of the alternatives and the best ranking of the alternatives are then the ones that have the highest likelihood of producing the given profile. This model aligns very well with our proposal and justifies the use of voting rules in the aggregation of classifiers' prediction. Moreover, MLEs give also a justification to the performance of ensembles where voting rules are used.


\section{VORACE}
\label{sec-vorace}

The main idea of VORACE (VOting with RAndom ClassifiErs) is that, given a sample, the output of each classifier can be seen as a ranking over the available classes, where the ranking order is given by the classifier's expected probability that the sample belongs to a class. Then a voting rule is used to aggregate these orders and declare a class as the "winner".
VORACE generates a profile of $n$ classifiers (where $n$ is an input parameter) by randomly choosing the type of each classifier among a set of predefined ones. For instance, the classifier type can be drawn between a decision tree or a neural network. For each classifier, some of its hyper-parameters values are chosen at random, where the choice of which hyper-parameters and which values are randomly chosen depends on the type of the classifier. When all classifiers are generated, they are trained using the same set of training samples.
For each classifier, the output is a vector with as many elements as the classes, where the $i$-th element represents the probability that the classifier assigns the input sample to the $i$-th class. Output values are ordered from the highest to the smallest one, and the output of each classifier is interpreted as a ranking over the classes, where the class with the highest value is the first in the ranking, then we have the class that has the second highest value in the output of the classifier, and so on. These rankings are then aggregated using a voting rule. The winner of the election is the class with the higher score. This corresponds to the prediction of VORACE.
Ties can occur when more than one class gets the same score from the voting rule. In these cases, the winner is elected using a tie-breaking rule, which chooses the candidate that is most preferred by the classifier with the highest validation accuracy in the profile.

\begin{example}
Let us consider a profile composed by the output vectors of three classifiers, say $y_1$, $y_2$ and $y_3$, over four candidates (classes) $c_1$, $c_2$, $c_3$ and $c_4$: $y_1 = [0.4,0.2,0.1,0.3]$, $y_2 = [0.1,0.3,0.2,0.4]$, and $y_3 = [0.4,0.2,0.1,0.3]$.
For instance, $y_1$ represents the prediction of the first classifier, which could predict that the input sample belongs to the first class with probability $0.4$, to the second class with probability $0.2$, to the third class with probability $0.1$ and to the fourth class with probability $0.3$.
From the previous predictions we can derive the correspondent ranked orders $x_1$, $x_2$ and $x_3$. For instance, from prediction $y_1$ we can see that the first classifier prefers $c_1$, then $c_4$, then $c_2$ and then $c_3$ is the less preferred class for the input sample. Thus we have:
$
x_1 = 
  \begin{bmatrix}
       c_1,           
       c_4,
       c_2,
       c_3 
  \end{bmatrix} 
\text{, }
x_2 = 
   \begin{bmatrix}
       c_4,
       c_2,
       c_3,
       c_1
  \end{bmatrix} 
\text{ and } 
x_3 =
  \begin{bmatrix}
       c_1,
       c_4,
       c_2,
       c_3
  \end{bmatrix} 
$.
Using Borda, class $c_1$ gets 6 points, $c_2$ gets 4 points, $c_3$ gets 1 point and $c_4$ gets 7 points. Therefore, $c_4$ is the winner, i.e. VORACE outputs $c_4$ as the predicted class. 
On the other hand, if we used Plurality, the winning class would be $c_1$, since it is preferred by 2 out of 3 voters.
\end{example}

Notice that this method does not need any knowledge of architecture, type, or parameters, of the individual classifiers.\footnote{Code available at \url{https://github.com/aloreggia/vorace/}.}

\section{Experimental Results}
\label{sec-exp}
\begin{table}[t!]
\centering
\resizebox{\linewidth}{!}{
\begin{tabular}{lccccc}  
\toprule
Dataset	& \#Examples	& \#Categorical  & \#Numerical  & Missing 	& \#Classes\\
\midrule
anneal			& 898 	 & 32 &	6 	&yes &6\\
autos			& 205 	 & 10 &	15 	&yes &7 \\
balance-s  		& 625 	 & 0  &	4	& no &3\\
breast-cancer	& 286 	& 9	&	0	&yes	&	2\\
breast-w		& 699 	& 0	&	9	&yes	& 	2\\
cars            & 1728  & 6 & 0     &no     &   4\\
credit-a		& 690 	& 9	&	6	&yes	&	2\\
colic			& 368 	& 15&	7	&yes	&	2\\
dermatology		& 366 	 & 33 &	1 	&yes	 &6\\
glass	 		& 214 	 & 0  &	9 	&no  &5\\
haberman		& 306 	& 0	&	3	&no		&	2\\
heart-statlog	& 270 	& 0	&	13	&no		&	2 \\
hepatitis		& 155 	& 13&	6	&yes	&	2\\
ionosphere		& 351 	& 34&	0	&no		&	2\\
iris			& 150 	 & 0  &	4 	&no  &3\\
kr-vs-kp        & 3196  & 0  & 36   &no &2\\
letter		 	& 20,000 & 0  &	16 	&no  &26\\
lymphogra 		& 148 	 & 15 &	3 	&no  &4\\
monks-3 		& 122 	& 6	&	0	&no		&	2 \\
spambase		& 4,601	& 0 &	57 	&no 	&	2\\
vowel			& 990 	 & 3  &	10 	&no  &11\\
wine 			& 178 	& 0	&	13	&no		&	3\\
zoo				& 101 	 & 16 &	1 	&no  &7\\
\bottomrule
\end{tabular}}
\caption{Description of the datasets.}
\label{tab:dataset-desc}
\end{table}
\begin{table*}[ht]
\footnotesize
\centering
\resizebox{\linewidth}{!}{
\begin{tabular}{lccccccc}  
\toprule
\# Voters 	 & Avg Profile	 & Borda  	 & Plurality	 &  Copeland 	 & Kemeny 	 & Sum 	 & Best C.  \\ 
\midrule
5   	 & 0.8599 (0.1021) 	 & 0.8864 (0.1043) 	 & 0.8885 (0.1052) 	 & 0.8885 (0.1051) 	 & \textbf{0.8886 (0.1050)} 	 & 0.8864 (0.1116) 	 & 0.8720 (0.1199) \\ 
7   	 & 0.8652 (0.0990) 	 & 0.8942 (0.0995) 	 & \textbf{0.8966 (0.1005)} 	 & 0.8965 (0.1007) 	 & \textbf{0.8966 (0.1007)} 	 & 0.8942 (0.1052) 	 & 0.8689 (0.1168) \\ 
10   	 & 0.8626 (0.0988) 	 & 0.8990 (0.0979) 	 & 0.9007 (0.0998) 	 & 0.9004 (0.1001) 	 & \textbf{0.9008 (0.1007)} 	 & 0.8985 (0.1050) 	 & 0.8667 (0.1196) \\ 
20   	 & 0.8615 (0.0965) 	 & 0.9015 (0.0968) 	 & \textbf{0.9043 (0.0977)} 	 & 0.9036 (0.0981) 	 & 0.9033 (0.0987) 	 & 0.8992 (0.1065) 	 & 0.8655 (0.1203) \\ 
40   	 & 0.8630 (0.0960) 	 & 0.9044 (0.0958) 	 & \textbf{0.9066 (0.0967)} 	 & 0.9060 (0.0968) 	 & 0.9058 (0.0969) 	 & 0.9006 (0.1050) 	 & 0.8651 (0.1183) \\ 
50   	 & 0.8633 (0.0957) 	 & 0.9044 (0.0962) 	 & \textbf{0.9068 (0.0970)} 	 & 0.9060 (0.0970) 	 & 0.9062 (0.0972) 	 & 0.8995 (0.1076) 	 & 0.8655 (0.1204) \\ 
\hline
Avg 	 & 0.8626 (0.0981) 	 & 0.8983 (0.0987) 	 &  \textbf{0.9006 (0.0998)} 	 & 0.9002 (0.0998) 	 & 0.9002 (0.1001) 	 & 0.8964 (0.1070) 	 & 0.8673 (0.1192)  \\

\bottomrule
\end{tabular}
}
\caption{Average F1-scores (and standard deviation), varying the number of voters, averaged over all datasets.}
\label{tab:avg-perf}
\end{table*}
\begin{table}[ht]
\centering
\footnotesize
\resizebox{\linewidth}{!}{
\begin{tabular}{lcccc}  
\toprule
Dataset 	 &   Majority 	 & Sum 	 & RF  	 &  XGBoost \\ \midrule
breast-cancer 	 & \textbf{0.7356 (0.0947)} 	 
& 0.7151 (0.0983) 	 &  0.7134 (0.0397) 	 & 0.7000 (0.0572) \\
breast-w 	 & 0.9645 (0.0133) 	 
& 0.9610 (0.0168) 	 &  \textbf{0.9714 (0.0143)} 	 & 0.9613 (0.0113) \\
colic 	 & 0.8587 (0.0367) 	 
& 0.8573 (0.0514) 	 &  0.8507 (0.0486) 	 & \textbf{0.8750 (0.0534)} \\
credit-a 	 & 0.8590 (0.0613) 	 
& 0.8478 (0.0635) 	 &  \textbf{0.8710 (0.0483)} 	 & 0.8565 (0.0763) \\
haberman 	 & 0.7337 (0.0551) 	 
& 0.6994 (0.0765) 	 &  \textbf{0.7353 (0.0473)} 	 & 0.7158 (0.0518) \\
heart-statlog 	 & 0.8070 (0.0699) 	 
& 0.7885 (0.0797) 	 &  \textbf{0.8259 (0.0621)} 	 & 0.8222 (0.0679) \\
hepatitis 	 & 0.8385 (0.0903) 	 
& 0.8377 (0.0955) 	 &  \textbf{0.8446 (0.0610)} 	 & 0.8242 (0.0902) \\
ionosphere 	 & \textbf{0.9435 (0.0348)} 	 
& 0.9366 (0.0344) 	 &  0.9344 (0.0385) 	 & 0.9260 (0.0427) \\
kr-vs-kp 	 & 0.9958 (0.0044) 	 
& \textbf{0.9960 (0.0044)} 	 &  0.9430 (0.0139) 	 & 0.9562 (0.0174) \\
monks-3 	 & 0.9182 (0.0712) 	 
& 0.9115 (0.0748) 	 &  \textbf{0.9333 (0.0624)} 	 & \textbf{0.9333 (0.0624)} \\
spambase 	 & \textbf{0.9416 (0.0105)} 	 
& 0.8801 (0.1286) 	 &  0.9100 (0.0137) 	 & 0.9294 (0.0112) \\
\midrule
Average & 0.8724 (0.0493) & 0.8574 (0.0658)	 & 0.8666 (0.0409)	& 0.8636 (0.0493)\\
\bottomrule
\end{tabular}
}
\caption{Performances on binary datasets: Average F1-scores (and standard deviation). Best performance in bold. On binary datasets, all the voting rules behave as majority voting rule. }
\label{tab:mix-results-bin}
\end{table}
\begin{table*}[ht]
\centering
\footnotesize
\resizebox{\linewidth}{!}{
\begin{tabular}{lccccccc}  
\toprule
Dataset 	 &   Borda 	 & Plurality 	 &  Copeland  	 & Kemeny	 & Sum 	 & RF  	 &  XGBoost \\ \midrule
anneal 	 & \textbf{0.9917 (0.0138)} 	 &  0.9876 (0.0200) 	 & 0.9876 (0.0200) 	 & 0.9880 (0.0194) 	 & 0.9894 (0.0174) 	 &  0.8471 (0.0122) 	 & 0.9912 (0.0080) \\
autos 	 & 0.8021 (0.0669) 	 &  0.7848 (0.0794) 	 & 0.7803 (0.0768) 	 & 0.7832 (0.0771) 	 & 0.8095 (0.0749) 	 &  0.6890 (0.0743) 	 & \textbf{0.8298 (0.0744)} \\
balance 	 & 0.9016 (0.0366) 	 &  \textbf{0.9208 (0.0311)}	 & 0.9069 (0.0292) 	 & 0.9082 (0.0297) 	 & 0.8911 (0.0376) 	 &  0.8561 (0.0540) 	 & 0.8578 (0.0441) \\
cars 	 & 0.9916 (0.0079) 	 &  0.9932 (0.0054) 	 & 0.9931 (0.0054) 	 & \textbf{0.9934 (0.0053)} 	 & 0.9931 (0.0048) 	 &  0.7928 (0.0300) 	 & 0.8935 (0.0266) \\
dermatology 	 & \textbf{0.9819 (0.0192)} 	 &  0.9769 (0.0206) 	 & 0.9769 (0.0206) 	 & 0.9766 (0.0209) 	 & 0.9783 (0.0196) 	 &  0.9699 (0.0256) 	 & 0.9755 (0.0189) \\
glass 	 & 0.9708 (0.0364) 	 &  0.9602 (0.0319) 	 & 0.9607 (0.0291) 	 & 0.9611 (0.0287) 	 & \textbf{0.9742 (0.0268)} 	 &  0.9535 (0.0295) 	 & 0.9719 (0.0313) \\
iris 	 & 0.9473 (0.0576) 	 &  0.9473 (0.0576) 	 & 0.9473 (0.0576) 	 & 0.9480 (0.0570) 	 & 0.9527 (0.0519) 	 &  0.9533 (0.0521) 	 & \textbf{0.9600 (0.0442)} \\
letter & 0.9311 (0.01)&  0.9590 (0.01)& 0.9545 (0.01) &- & 0.9627 (0.01) & 0.6044 (0.01) & 0.8832 (0.01) \\
lymphography 	 & 0.8461 (0.0983) 	 &  \textbf{0.8630 (0.0851)} 	 & 0.8624 (0.0843) 	 & 0.8604 (0.0875) 	 & 0.8529 (0.0925) 	 &  0.8586 (0.0691) 	 & 0.8519 (0.0490) \\
vowel 	 & 0.9476 (0.0232) 	 &  \textbf{0.9862 (0.0110)} 	 & 0.9860 (0.0116) 	 & 0.9860 (0.0114) 	 & \textbf{0.9862 (0.0119)} 	 &  0.7333 (0.0323) 	 & 0.8323 (0.0333) \\
wine 	 & 0.9656 (0.0537) 	 &  0.9789 (0.0331) 	 & 0.9783 (0.0342) 	 & 0.9783 (0.0342) 	 & 0.9806 (0.0380) 	 &  \textbf{0.9889 (0.0222)} 	 & 0.9611 (0.0558) \\
zoo 	 & 0.9550 (0.0497) 	 &  0.9550 (0.0517) 	 & 0.9560 (0.0496) 	 & 0.9590 (0.0492) 	 & 0.9500 (0.0640) 	 &  0.9500 (0.0500) 	 & \textbf{0.9700 (0.0640)} \\

\midrule
Average & 0.9365 (0.0421)	& 0.9413 (0.0388)   & 0.9396 (0.0380)	& 0.9402 (0.0382)	& 0.9416 (0.0399)	&0.8720 (0.0410)	&0.9177 (0.0409)\\
\bottomrule
\end{tabular}
}
\caption{Performances on multiclass datasets: Average F1-scores (and standard deviation). Best performance in bold. }
\label{tab:mix-results-mul}
\end{table*}
We considered 23 datasets from the UCI \cite{Newman+Hettich+Blake+Merz:1998} repository. Table \ref{tab:dataset-desc} gives a brief description of these datasets in terms of number of examples, number of features (where some features are categorical and others are numerical), whether there are missing values for some features, and number of classes.
To generate the individual classifiers, we use three 
classification algorithms: Decision Trees (DT), Neural Networks (NN), and Support Vector Machines (SVM). 

{\bf Neural networks} are generated by choosing $2$, $3$ or $4$ hidden layers with equal probability. 
For each hidden layer, the number of nodes is sampled geometrically in the range $[A,B]$, which means computing $\lfloor(e^x)\rfloor$ where $x$ is drawn uniformly in the interval $[\log(A),\log(B)]$ \cite{BergstraB12}. We choose $A=16$ and $B=128$.
The activation function is chosen with equal probability between the rectifier function $f(x)=max(0,x)$ and the hyperbolic tangent function.
The maximum number of epochs to train each neural network is set to $100$. An early stopping callback is used to prevent the training phase to continue even when the accuracy is not improving and we set the patience parameter to $p=5$. 
Batch size value is adjusted to respect the size of the dataset: given a training set $T$ with size $l$, 
the batch size is set to $b=2^{\ceil{log_2(x)}}$ where $x=\frac{l}{100}$.

{\bf Decision trees} are generated by choosing between the \textit{entropy} and \textit{gini} criteria with equal probability, and with a maximal depth uniformly sampled in $[5,25]$.

{\bf SVMs} are generated by choosing randomly between the \textit{rbf} and \textit{poly} kernels. For both types, the $C$ factor is drawn geometrically in $[2^{-5},2^{5}]$. If the type of the kernel is poly, the coefficient is sampled at random in $[3,5]$. For rbf kernel, the gamma parameter is set to auto. 

We used the average F1-score of a classifier ensemble as the evaluation metric, for all 23 different data sets, since the F1-score is a better measure to use if we need to seek a balance between Precision and Recall.
For each dataset, we train and test the ensemble method with a 10-fold cross validation process. Additionally, for each dataset, experiments are performed 10 times, leading to a total of 100 runs for each method over each dataset. This is done to ensure greater stability. The voting rules considered in the experiments are Plurality, Borda, Copeland and Kemeny.

\cri{In order to compute the Kemeny consensus, we leverage the implementation of the Kemeny method for rank aggregation of incomplete rankings with ties that is available with the Python package named corankco\footnote{The package is available at \texttt{https://pypi.org/project/corankco/}.}. The package provides several methods for computing a Kemeny consensus. Finding a Kemeny consensus is computationally hard, especially when the number of candidates grows. In order to ensure the feasibility of the experiments, we compute a Kemeny consensus using the exact algorithm with ILP Cplex when the number of classes $|C|\leq 5$, otherwise we employed the consensus computation with a heuristic (see package documentation for further details).}
We compare the performance of VORACE to 1) the average performance of a profile of individual classifiers, 2) the performance of the best classifier in the profile, 3) two state-of-the-art methods (Random Forest and XGBoost), and 4) the {\it Sum} method (also called {\it weighted averaging}). The {\it Sum} method  computes $x_{j}^{\text{Sum}}=\sum_i^n x_{j,i}$ for each individual classifier $i$ and for each class $j$, where $x_{j,i}$ is the probability that the sample belongs to class $j$ predicted by classifier $i$. The winner is 
the one with the maximum value in the sum vector: $\arg\max x_{j}^{\text{Sum}}$. 
We did not compare VORACE to a more sophisticated version of Sum, such as {\it conditional averaging}, since they are not applicable in our case, requiring additional knowledge of the domain which is out of the scope of our work.
Both Random Forest and XGBoost classifiers are generated with the same number of trees as the number of classifiers in the profile, all the remaining parameters are generated using default values.
We did not compare to {\it stacking} because it would require to manually identify the correct structure of the sequence of classifiers in order to obtain competitive results. An optimal structure (i.e., a definition of a second level meta-classifier) can be defined by an expert in the domain at hand \cite{breiman1996stacked}, and this is out of the scope of our work.

To study the accuracy of our method, we performed three kinds of experiments: 1) varying the number of individual classifiers in the profile and averaging the performance over all datasets, 2) fixing the number of individual classifiers and analyzing the performance on each dataset and 3) considering the introduction of more complex classifiers as base classifiers for VORACE. Since the first experiment shows that the best accuracy of the ensemble occurs when $n=50$, we use only this size for the second and third experiments.

\subsection{Experiment 1: Varying the number of voters in the ensemble}

The aim of the first experiment is twofold: on one hand, we want to show that increasing the number of classifiers in the profile leads to an improvement of the performance. On the other hand, we want to show the effect of the aggregation on performance, compared with the best classifier in the profile and with the average classifier's performance. To do that, we first evaluate the overall average accuracy of VORACE varying the number $n$ of individual classifiers in the profile. Table \ref{tab:avg-perf} presents the performance of each ensemble for different numbers of classifiers, specifically $n \in \{5,7,10,20,40,50\}$. Plurality, Copeland, and Kemeny voting rules have their best accuracy for VORACE  when $n=50$. We set the system to stop the experiment after the time limit of one week, this is why we stop when $n=50$. We are planning to run experiments with larger time limits in order to understand whether the system shows that the effect of the profile's size diminishes at some point. 
In Table \ref{tab:avg-perf}, we report the F1-score and the standard deviation of VORACE with the considered voting rules. The last line of the table presents the average F1-score for each voting rule. The dataset ‘‘letter" was not considered in this test.

Increasing the number of classifiers in the ensemble, all the considered voting rules show an increase of the performance, specifically the higher the number of the classifiers the higher the F1-score of VORACE.

\cri{However, in Table~\ref{tab:avg-perf} we can observe that the performance is slightly incremental when we increase the number of classifiers. This is due to the fact that in this particular experiment the accuracy of every single classifier is usually very high (i.e., $p \geq 0.8$), thus the ensemble has a reduced contribution to the aggregated result. In general this is not the case, especially when we have to deal with ``harder'' datasets where the accuracy $p$ of single classifiers is lower.
In Section \ref{sec:tech-analysis}, we will explore this case and we will see that 
the number of classifiers has a greater impact on the accuracy of the ensemble when the accuracy of the classifiers on average is low (e.g., $p \leq 0.6$).
}

\cri{Moreover, it is worth noting that the computational cost of the ensemble (both training and testing) increases linearly with the number of classifiers in the profile. Thus, it is convenient to consider more classifiers, especially when the accuracy of the single classifiers is poor.}
Thus, overall, the increase in the number of classifiers has a positive effect on the performance of VORACE, as expected given the theoretical analysis in Section \ref{sec:tech-analysis}\footnote{However, the experiments do not satisfy the independence assumption of the theoretical study}.

For each voting rule, we also compared VORACE 
to the average performance of the individual classifiers and the best classifier in the profile, to understand if VORACE is better than 
the best classifier, or if it is just better than the average classifiers' accuracy (around $0.86$). 
In Table \ref{tab:avg-perf} we can see that VORACE always behaves better than both the best classifier and the profile's average. Moreover, it is interesting to observe that Plurality performs better on average than more complex voting rules like Borda and Copeland. 


\subsection{Experiment 2: Comparing with existing methods}


For the second experiment, we set $n=50$ and we compare VORACE (using Majority, Borda, Plurality, Copeland, and Kemeny) with Sum, Random Forest (RF), and XGBoost in each dataset. Table \ref{tab:mix-results-bin} reports the performances of VORACE on binary datasets where all the voting rules collapse to Majority voting. VORACE performances are very close to the state-of-the-art. We try to use Kemeny on the dataset ‘‘letter" but it exceeds the time limit of one week and thus it was not possible to compute the average. In order to make the average values comparable (last row of Table \ref{tab:mix-results-mul}),  performances on the dataset ‘‘letter" were not considered in the computation of the average values for the other methods.
Table \ref{tab:mix-results-mul} reports the performances on datasets that have multiple classes: when the number of classes increases VORACE is still stable and behaves very similarly to the state-of-the-art. 
The similarity among the performances is promising for the system. Indeed, RandomForest and XGBoost reach better performances on some datasets and they can be improved on over by optimizing their hyperparameters. But, this experiment shows that it is possible to reach very similar performances using a very simple method as VORACE is. This means that usage of VORACE does not require any optimization of hyperparameters whether it is done manually or automatically.
The importance of this property is corroborated by a recent line of work by  \cite{StrubellGM19} that suggests how industry and academia should focus their efforts on developing tools that reduce or avoid hyperparameters' optimization, resulting in simpler methods that are also more sustainable in terms of energy and time consumption.

Moreover, Plurality is both more time and space efficient since it needs a smaller amount of information: for each classifier it just needs the most preferred candidate instead of the whole ranking, contrarily to other methods such as Sum.
We also performed two additional variants of these experiments, one with a weighted version of the voting rules (where the weights are the classifiers' validation accuracy), and the other one by training each classifier on different portions of the data in order to increase the independence between them. In both experiments, the results are very similar to the ones reported here. 

\subsection{Experiment 3: Introducing complex
classifiers in the profile}
The goal of the third experiment is to understand whether using complex classifiers in the profile (such as using an ensemble of ensembles) would produce better final performances.
For this purpose, we compared VORACE with standard base classifiers (described in Section \ref{sec-vorace}) with three different versions of VORACE with complex base classifiers: 1) VORACE with only Random Forest 2) VORACE with only XGBoost and 3) VORACE with Random Forest, XGBoost and standard base classifiers (DT, SVM, NN). 

For simplicity, we used the Plurality voting rule, since it is the most efficient method and it is one of the voting rules that gives better results.
We fixed the number of voters in the profiles to $50$ and we selected the parameters for the simple classifiers for VORACE as described at the beginning of Section \ref{sec-exp}. For Random Forest, parameters were drawn uniformly among the following values\footnote{Parameters' names and values refer to the Python's modules: \texttt{RandomForestClassifier} in \texttt{sklearn.ensemble} and \texttt{xgb} in \texttt{xgboost}.}:
{\it bootstrap} between {\it True} and {\it False}, 
{\it max\_depth} in $[10, 20, \dots, 100, None]$,
{\it max\_features} between $[auto, sqrt]$, 
{\it min\_samples\_leaf} in $[1, 2, 4]$, 
{\it min\_samples\_split} in $[2, 5, 10]$, and 
{\it n\_estimators} in $[10,20,50,100,200]$.  
For XGBoost instead the parameters were drawn uniformly among the following values:
{\it max\_depth} in the range $[3,25]$, 
{\it n\_estimators} equals the number of classifiers, 
{\it subsample} in $[0,1]$, and 
{\it colsample\_bytree} in $[0,1]$. 
The results of the comparison between the different versions of VORACE are provided in Table \ref{tab:mix-results2}. We can observe that the performance of VORACE (column "Majority" of Table \ref{tab:mix-results-bin} and column "Plurality" of Table \ref{tab:mix-results-mul}) is not significantly improved by using more complex classifiers as a base for the profile. 
It is interesting to notice the effect of VORACE on the aggregation of RF with respect to a single RF. Comparing the results in Table \ref{tab:mix-results-bin} and \ref{tab:mix-results-mul} (RF column) with results in Table \ref{tab:mix-results2} (VORACE with only RF column), one can notice that RF is positively affected by the aggregation on many datasets (on all the datasets the improvement is on average $5$\%), especially on those with multiple classes. Moreover, the improvement is significant in many of them: i.e. on ``letter'' dataset we have an improvement of more than $35$\%.
This effect can be explained by the random aggregation of trees used by the RF algorithm, where the goal is to reduce the variance of the single classifier. In this sense, a principled aggregation of different RF models (as the one in VORACE) is a correct way to boost the final performance: distinct RF models act differently over separate parts of the domain, providing VORACE with a good set of weak classifiers -- see Theorem \ref{th:overlap}.

We saw in this section that this more complex version of VORACE does not provide any significant advantage, in terms of performance, compared with the standard one. To conclude, we thus suggest using VORACE in its standard version without adding complexity to the base classifiers.

\begin{table}[!h]
\centering

\begin{tabular}{lccc}  
\toprule
\multirow{2}{*}{dataset}
& VORACE with	 &  VORACE with &  VORACE with	\\
& RF \& XGBoost 	 &   only RF &   only XGBoost 	\\
\midrule
anneal 	 	 &  0.9937 (0.01) 	  &  0.9921 (0.01) 	 & 0.9893 (0.01) 	\\
autos 	 	 &  0.8095 (0.09) 	  &  0.7969 (0.10) 	 & 0.7916 (0.08) 	 \\
balance 	 	 &  0.8998 (0.02) 	  &  0.8456 (0.03) 	 & 0.8040 (0.04)  \\
breast-cancer* 	 	 &  0.7573 (0.04) 	  &  0.7485 (0.06) 	 & 0.7394 (0.06)  \\ 	
breast-w* 	 	 &  0.9654 (0.02) 	  &  0.9744 (0.02) 	 & 0.9605 (0.03)  \\
cars 	 	 &  0.9887 (0.01) 	  &  0.9547 (0.01) 	 & 0.9044 (0.05) 	\\ 	
colic*	 	 &  0.8668 (0.04) 	  &  0.8766 (0.04) 	 & 0.8638 (0.04) 	 \\ 	
credit-a* 	 	 &  0.8737 (0.03) 	  &  0.8691 (0.03) 	 & 0.8712 (0.03)  \\
dermatology 	 	 &  0.9749 (0.02) 	  &  0.9765 (0.02) 	 & 0.9805 (0.02)  \\ 	
glass 	 	 &  0.9761 (0.03) 	  &  0.9740 (0.04) 	 & 0.9770 (0.03) 	 \\
haberman* 	 	 &  0.7338 (0.04) 	  &  0.7168 (0.04) 	 & 0.7286 (0.02)  \\
heart-statlog* 	 	 &  0.8315 (0.09) 	  &  0.8352 (0.09) 	 & 0.8248 (0.08)  \\ 	
hepatitis*	 	 &  0.8215 (0.07) 	  &  0.8091 (0.05) 	 & 0.8105 (0.08)  \\ 	
ionosphere* 	 	 &  0.9349 (0.04) 	  &  0.9272 (0.05) 	 & 0.9347 (0.04)  \\
iris 	 	 &  0.9627 (0.05) 	  &  0.9593 (0.04) 	 & 0.9593 (0.05) 	\\
kr-vs-kp* 	 	 &  0.9953 (0.00) 	  &  0.9869 (0.01) 	 & 0.9892 (0.01)  \\
letter 	 	 &  0.9632 (0.01) 	  &  0.9622 (0.01) 	 & 0.9265 (0.01) 	 \\
lymphography 	 	 &  0.8700 (0.10) 	  &  0.8306 (0.15) 	 & 0.8412 (0.14)  \\ 	
monks-3* 	 	 &  0.9156 (0.07) 	  &  0.9340 (0.06) 	 & 0.9037 (0.07)  \\ 	
spambase* 	 	 &  0.9437 (0.01) 	  &  0.9439 (0.01) 	 & 0.9337 (0.01)  \\ 	
vowel 	 	 &  0.9834 (0.01) 	  &  0.9691 (0.02) 	 & 0.9086 (0.03) 	 \\ 	
wine 	 	 &  0.9851 (0.03) 	  &  0.9764 (0.04) 	 & 0.9796 (0.04) 	 \\ 	
zoo 	 	 &  0.9535 (0.05) 	  &  0.9589 (0.05) 	 & 0.9231 (0.06) 	 \\
\midrule
Average & {\bf 0.9130 (0.04)} & 0.9051 (0.04) & 0.8933 (0.04)\\
\bottomrule
\end{tabular}
\caption{Average F1-scores (and standard deviation). * denotes binary datasets.}
\label{tab:mix-results2}
\end{table}

In other experiments, we also see that the probability of choosing the correct class decreases if the number of classes increases. This means that the task becomes more difficult with a larger number of classes.

\section{Theoretical analysis: Independent classifiers with the same accuracy}
\label{sec:tech-analysis}

In this section we theoretically analyze the probability to predict the correct label/class of our ensemble method.

Initially, we consider a simple scenario with $m$ classes (the candidates) and a profile of $n$ independent classifiers (the voters), where each classifier has the same probability $p$ of correctly classifying a given instance. The independence assumption hardly fully holds in practice, but it is a natural simplification (commonly adopted in literature) used for the sake of analysis.

We assume that the system uses the Plurality voting rule. This is justified by the fact that Plurality provides better results in our experimental analysis (see Section \ref{sec-exp}) and thus it is the one we suggest to use with VORACE. Moreover, Plurality has also the advantage to require very little information from the individual classifiers and also being computationally efficient. 

We are interested in computing the probability that VORACE chooses the correct class. 
This probability corresponds to the accuracy of VORACE when considering the single classifiers as black boxes, i.e., knowing only their accuracy without any other information. The result presented in the following theorem is especially powerful because it shows a closed formula that only requires for the values of $p$, $m$, and $n$ to be known.
\begin{theorem}\label{prop:1}
The probability of electing the correct class $c^*$, among $m$ classes, with a profile of $n$ classifiers, each one with accuracy $p \in [0,1]$ , using Plurality is given by:
\begin{equation}\label{eq:prob_exact}
\mathcal{T}(p) = \frac{1}{K}(1-p)^n \sum_{i=\lceil \frac{n}{m} \rceil}^{n} \varphi_i (n-i)!  \binom{n}{i} \left( \frac{p}{1-p} \right) ^i
\end{equation}
where $\varphi_i$ is defined as the coefficient of the monomial $x^{n-i}$ in the expansion of the following generating function:
$$
\mathcal{G}^m_i(x)=\left( \sum_{j=0}^{i-1} \frac{x^j}{j!} \right) ^{m-1}
$$
and $K$ is a normalization constant defined as:
$$
K=\sum_{j=0}^{n} \binom{n}{j} p^j(m-1)^{n-j} (1-p)^{n-j} ~.
$$
\end{theorem}
\begin{proof}
The formula can be rewritten as:
$$
\mathcal{T}(p) = \frac{1}{K}\sum_{i=\lceil \frac{n}{m} \rceil}^{n} \binom{n}{i}  p^i  \varphi_i(n-i)!(1-p)^{n-i}
$$
and corresponds to the sum of the probability of all the possible different profiles votes that elect $c^*$. We perform the sum varying $i$, an index that indicates the number of classifiers in the profile that vote for the correct label $c^*$. This number is between $\lceil \frac{n}{m} \rceil $ (since if $i < \lceil \frac{n}{m} \rceil$ that profile cannot elect $c^*$) and $n$ where all the classifiers vote for $c^*$. 
The binomial 
factor
expresses the number of possible positions, in the ordered profile of size $n$, of $i$ classifiers that votes for $c^*$. This is multiplied by the probability of these classifiers to vote $c^*$, that is $p^i$. 
The factors
$ \varphi_i(n-i)! $
correspond the number of possible combinations of votes of the $n-i$ classifiers (on the other candidates different from $c^*$) that ensure the winning of $c^*$.
This is computed as the number of possible combinations of $n-i$ objects extracted from a set $(m-1)$ objects, with a bounded number of repetitions (bounded by $i-1$ to ensure the winning of $c^*$).
The formula to use for counting the number of combinations of $D$ objects extracted from a set $A$ objects, with a bounded number of repetitions $B$, is: $\varphi_i D!$. In our case $A=m-1$ is the number of objects, $B=i-1$ is the maximum number of repetitions and $D=n-i$ the positions to fill and $ \varphi_i$ is the coefficient of $x^D$ in the expansion of the following generating function:
$$
\left( \sum_{j=0}^{B} \frac{x^j}{j!} \right) ^{A} \xRightarrow[B=i-1]{A=m-1}  \left( \sum_{j=0}^{i-1} \frac{x^j}{j!} \right) ^{m-1}= \mathcal{G}^m_i(x).
$$
Finally, the factor $(1-p)^{n-i}$ is the probability that the remaining $n-i$ classifiers do not elect $c^*$.
\qed
\end{proof}

\cri{For the sake of comprehension, we give an example that describes the computation of the probability of electing the correct class $c^*$, as formalized in Theorem \ref{prop:1}.}

\begin{example}
\cri{Considering an ensemble with 3 classifiers (i.e., $n=3$), each one with accuracy $p=0.8$. The number of classes in the dataset is $m=4$. The probability of choosing the correct class $c^*$ is given by the formula in Theorem~\ref{prop:1}. Specifically:}
$$\mathcal{T}(p) = (1 - 0.8)^{3} \frac{1}{K}\sum_{i=1}^{3} \varphi_i \cdot (3-i)! \binom{3}{i}  \left( \frac{0.8}{1-0.8} \right) ^i
$$
where $K = 1.728$ . In order to compute the value of each $\varphi_i$, we have to compute the coefficient of $x^{3-i}$ in the expansion of the generating function $\mathcal{G}_i^4(x)$.

\begin{description}
    \item[For $i=1$:] We have
$\mathcal{G}_1^4(x) = 1$, and
we are interested in the coefficient of $x^{n-i} = x^2$, thus $\varphi_1 = 0$.

\item[For $i=2$:] We have
$\mathcal{G}_2^4(x) = 1 + 3 x + 3 x^2 + x^3$, and
we are interested in the coefficient of $x^{n-i} = x^1$, thus $\varphi_2 = 3$.

\item[For $i=3$:] We have 
$\mathcal{G}_3^4(x) = 1 + 3 x + \frac{9}{2} x^2 + 4 x^3 + \frac{9}{4}x^4 + \frac{3}{4}x^5 + \frac{1}{8}x^6$, and
we are interested in the coefficient of $x^{n-i} = x^0$, thus $\varphi_3 = 1$.
\end{description}
\cri{
We can now compute the probability $\mathcal{T}(p)$:
}
\begin{align*}
\mathcal{T}(p) & =   \frac{0.008}{1.728} \cdot ( 
\varphi_1 \cdot (2)! \binom{3}{1} \cdot 4  + 
\varphi_2 \cdot (1)! \binom{3}{2} \cdot (4)^2  +
\varphi_3 \cdot (0)! \binom{3}{3} \cdot (4)^3 
) \\
& = 0.963.
\end{align*}
The result says that VORACE with 3 classifiers (each one with accuracy $p=0.8$) has a probability of $0.963$ of choosing the correct class $c^*$.
\end{example}

\cri{It is worth noting that $\mathcal{T}(p)=1$ 
when $p=1$, meaning that, when all the classifiers in the ensemble always predict the right class, our ensemble method always outputs the correct class as well \footnote{Formula \ref{eq:prob_exact} is equal to $1$ for $p=1$ because all the terms of the sum are equal to zero except the last term for $i=n$ ($K=1$ and $ \varphi_i(0)=1$ as well). This is equal to $1$ because we have $(1-p)^0=0^0$ and by convention $0^0=1$ when we are considering discrete exponents.}. Moreover, $\mathcal{T}(p)=0$ in the symmetric case in which $p=0$, that is when all the classifiers always predict a wrong class.}

\cri{Note that the independence assumption considered above is in line with previous studies (e.g., the same assumption is made in \cite{condorcet_th_jury,youngcondorcet}) and
it is a necessary simplification to obtain a closed formula to compute $\mathcal{T}(p)$. Moreover, in a realistic scenario, $p$ can be interpreted as representing the lower bound of the accuracy of the classifiers in the profile. It is easy to see that under this interpretation the value of $\mathcal{T}(p)$ as well represents a lower bound of the probability of electing the correct class $c^*$, given $m$ available classes, and a profile of $n$ classifiers.}

Although this theoretical result holds in a restricted scenario and with a specific voting rule, as we already noticed in our experimental evaluation in Section \ref{sec-exp}, the probability of choosing the correct class is always greater than or equal to each individual classifiers' accuracy. 

It is worth noting that the scenario considered above is similar to the one analyzed in the {\it Condorcet Jury Theorem} \cite{condorcet_th_jury},
which states that in a scenario with two candidates where each voter has probability $p>\frac{1}{2}$ to vote for the correct candidate, the probability that the correct candidate is chosen goes to $1$ as the number of votes goes to infinity. Some restrictions imposed by this theorem are partially satisfied also in our scenario: some voters (classifiers) are independent on each other (those that belong to a different classifier's category), since we generate them randomly. 
\cri{However, Theorem~\ref{prop:1} does not immediately follow from this result. Indeed, it represents a generalization because some of the Condorcet restrictions do not hold in our case, specifically:}
1) 2-class classification task does not hold, since VORACE can be used also with more than 2 classes; 
2) classifiers are generated randomly, thus we cannot ensure that the accuracy $p>\frac{1}{2}$, especially with more than two classes.
This work has been reinterpreted first by \cite{youngcondorcet} and successively extended by \cite{nitzanparoush} and \cite{shapleygrofman}, considering the cases in which the agents/voters have different $p_i$. \cri{However, the focus of these works is fundamentally different from ours, since their goal is to find the optimal decision rule that maximizes the probability that a profile elects the correct class.}

Given the different conditions of our setting, we cannot apply the Condorcet Jury Theorem, or the works cited above, as such. 
However, in Section~\ref{sec:Condorcet} we will formally see that considering $m= 2$, our formulation enforces the results stated by Condorcet Jury Theorem. 

Moreover, our work is in line with the analysis regarding maximum likelihood estimators (MLEs) for r-noise models \cite{handbooksocialchoice,ConitzerMLE}.  
An {\it r-noise model} is a noise model for ranking over a set of candidates, i.e., a family of probability distributions in the form $P(\cdot| u)$, where $u$ is the correct preference. This means that an r-noise model describes a voting process where there is a ground truth about the collective decision, although the voters do not know it.  In this setting, a MLE is a preference aggregation function $f$ that maximizes the product of the probabilities $P(v_i | u),\,i=1,\dots,n$ for a given voters' profile $R=(v_1,...,v_n)$. Finding a suitable $f$
corresponds to our goal.

MLEs for r-noise models have been studied in details by \citet{ConitzerMLE} assuming the noise is independent across votes. This corresponds to our preliminary assumption of the independence of the base classifiers. The first result in \cite{ConitzerMLE} states that given a voting rule, there always exists a r-noise model such that the voting rule can be interpreted as a MLE (see Theorem 1 in \cite{ConitzerMLE}). In fact, given an appropriate r-noise model, any scoring rule is a maximum likelihood estimator for winner under i.i.d. votes. 
Thus, for a given input sample, we can interpret the classifiers rankings as a permutation of the true ranking over the classes and the voting rule (like Plurality and Borda) used to aggregate these rankings as a MLE for an r-noise model on the original classification of the examples.
However, to the best of our knowledge, providing a closed formulation (i.e., considering only the main problem's parameters $p$, $m$ and $n$, and without having any information on the original true ranking or the noise model) to compute the probability of electing the winner (as provided in our Theorem \ref{prop:1}) for a given profile using Plurality is a novel and valuable contribution (see discussion on the attempts existing in the literature to define the formula in Section \ref{sec:voting_ensemble}).
We remind the reader that in our learning scenario the formula in Theorem \ref{prop:1} is particularly useful because it computes a lower bound on the accuracy of VORACE (that is, the probability that VORACE selects the correct class) when knowing only the accuracy of the base classifiers, considering them as black boxes.


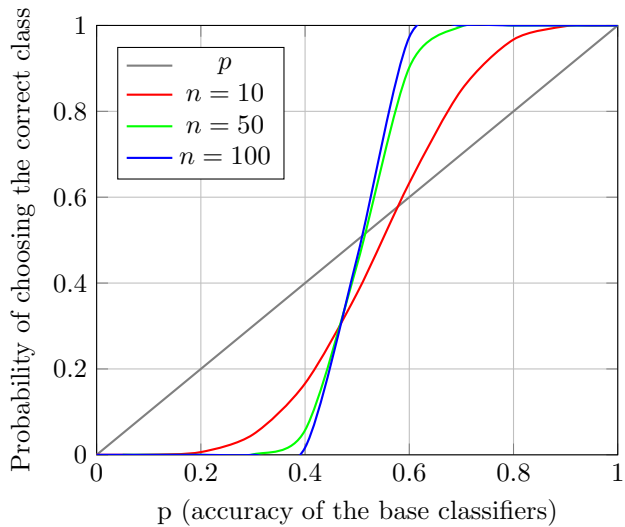
\begin{figure}
\centering
\begin{tikzpicture}
\normalsize
\begin{axis}[legend style={at={(0.2,0.95)},
		anchor=north,legend columns=1}, xmin=0,xmax=1,ymin=0, ymax=1, xlabel=p (accuracy of the base classifiers), ylabel=Probability of choosing the correct class,xmajorgrids, ymajorgrids,every axis plot/.append style={thick}]
\addplot+[gray,mark=,solid,smooth] 
coordinates
{
(0,0)
(0.05,0.05)
(0.1,0.1)
(0.15,0.15)
(0.2,0.2)
(0.25,0.25)
(0.3,0.3)
(0.35,0.35)
(0.4,0.4)
(0.45,0.45)
(0.5,0.5)
(0.55,0.55)
(0.6,0.6)
(0.65,0.65)
(0.7,0.7)
(0.75,0.75)
(0.8,0.8)
(0.85,0.85)
(0.9,0.9)
(0.95,0.95)
(1,1)
};
\addplot+[red,mark=, solid,smooth] 
coordinates
{
(0,0)
(0.05,0.0)
(0.1,0.000146903)
(0.2,0.00636938)
(0.3,0.047349)
(0.4,0.166239)
(0.5,0.376953)
(0.6,0.633103)
(0.7,0.849732)
(0.8,0.967207)
(0.9,0.998365)
(0.95,0.999936)
(1,1)
};
\addplot+[green,mark=,solid,smooth] 
coordinates
{
(0.05,0.0)
(0.1,0.0)
(0.2, 4.9241e-7)
(0.3,0.000933179)
(0.4,0.0573438)
(0.5,0.443862)
(0.6,0.902193)
(0.7,0.99763)
(0.8,0.999998)
(0.9,1.0)
(0.95,1.0)
(1,1)
};
\addplot+[blue,mark=,solid,smooth] 
coordinates
{
(0,0)
(0.05,0.0)
(0.1,0.0)
(0.2,5.17989e-12)
(0.3,9.03469e-6)
(0.4,0.0167617)
(0.5,0.460205)
(0.6,0.972901)
(0.7,0.999978)
(0.8,1.0)
(0.9,1.0)
(0.95,1.0)
(1,1)
};
\legend{$p$, $n=10$, $n=50$, $n=100$}
\end{axis}
\end{tikzpicture}
\caption{Probability of choosing the correct class $c^*$ varying the size of the profile $n \in \{10,50,100\}$ and keeping $m$ constant to 2, where each classifier has the same probability $p$ of classifying a given instance correctly.}
\label{fig:formula_p_n}
\end{figure}
More precisely, we analyze the relationship between the probability of electing the winner (i.e., Formula~\ref{eq:prob_exact}) and the accuracy of each individual classifier $p$.
Figure~\ref{fig:formula_p_n}\footnote{\cri{Figure~\ref{fig:formula_p_n} has been created by grid sampling the values of $p \in [0,1]$ with step $0.05$ and by performing an exact computation of the value of $\mathcal{T}(p)$ for each specific value of $p$ in the sampling set with $n \in \{10,50,100\}$ and $m=2$. We then connected these values with the smoothing algorithm of \emph{TikZ} package.}} shows the probability of choosing the correct class varying the size of the profile $n \in \{10,50,100\}$ and keeping $m=2$.
We see that, 
by augmenting the size of the profile $n$, the probability that the ensemble chooses the right class grows as well. 
\cri{However, the benefit is just incremental when base classifiers have high accuracy. We can see that when $p$ is high we reach a plateau where $\mathcal{T}(p)$ is very close to $1$ regardless of the number of classifiers in the profile. In a realistic scenario, having a high baseline accuracy in the profile is not to be expected, especially when we consider ``hard'' datasets and randomly generated classifiers.
In these cases (when the accuracy of the base classifiers in average is low), the impact of the number of classifiers is more evident (for example when $p = 0.6$).}

Thus, if $p>0.5$ and $n$ tends to infinity, then it is beneficial to use a profile of classifiers. This is in line with the result of the {\it Condorcet Jury Theorem}.


\section{Theoretical analysis: comparison with Condorcet Jury Theorem}\label{sec:Condorcet}
In this section we prove how, for $m=2$, Formula \ref{eq:prob_exact} enforces the results stated in the Condorcet Jury Theorem~\cite{condorcet_th_jury} (see Section~\ref{sec:tech-analysis} for the Condorcet Jury Theorem statement). 
Notice, as for Theorem 1, the adopted assumptions likely do not fully hold in practice, but are natural simplifications used for the sake of analysis.
Specifically, we need to prove the following theorem.

\begin{theorem}\label{thm:limit}
The probability of electing the correct class $c^*$, among $2$ classes, with a profile of an infinite number of classifiers, each one with accuracy $p \in [0,1]$, using Plurality, is given by:
\begin{equation} \label{eq:limit}
\lim_{n\to\infty}\mathcal{T}(p) = 
\begin{cases}
0 & p<0.5\\
0.5 & p=0.5\\
1  & p>0.5
\end{cases}
\end{equation}
\end{theorem}

In Figure~\ref{fig:limit} we can see a visualization of the function $\mathcal{T}(p)$ when  $n \rightarrow \infty$, as described in Theorem~\ref{thm:limit}. In what follows we will prove this by showing that the function $\mathcal{T}(p)$ is monotonic increasing and when $n \rightarrow \infty$ is equal to $0$.

\begin{figure}[h!]
\centering
\begin{tikzpicture}
\begin{axis}[legend style={at={(0.225,1.15)},
		anchor=north,legend columns=1}, xmin=0,xmax=1,ymin=0, ymax=1, xlabel=p (accuracy of the base classifiers), ylabel=Probability of choosing the correct class,xmajorgrids, ymajorgrids]
\draw [dashed] (0.5,0) -- (0.5,1);
\addplot[line width=3pt,blue,samples at={0,0.5}] {0};
\addplot[line width=3pt,blue,samples at={0.5,1}] {1};
\addplot[ultra thick,blue,mark=*,mark options={fill=white},samples at={0.5,0.5}] {0};
\addplot[ultra thick,blue,mark=*,mark options={fill=white},samples at={0.5,0.5}] {1};
\addplot[ultra thick,blue,mark=*,samples at={0.5,0.5}] {0.5};
\legend{$\mathcal{T}(p)$ for  $n \rightarrow \infty$}
\end{axis}
\end{tikzpicture}
\caption{The probability of electing the correct class $c^*$, among $2$ classes, with a profile of an infinite number of classifiers ($n \to \infty$), each one with accuracy $p \in [0,1]$, using Plurality.}
\label{fig:limit}
\end{figure}
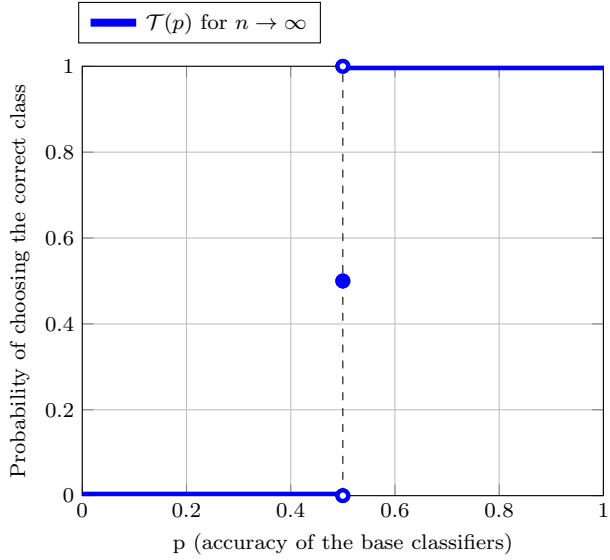

Firstly, we find an alternative, more compact, formulation for $\mathcal{T}(p)$ in the case of binary datasets (only two alternatives/candidates, i.e., $m=2$) in the following Lemma.
\begin{lemma}\label{lemma:t_p_sum}
The probability of electing the correct class $c^*$, among $2$ classes, with a profile of $n$ classifiers, each one with accuracy $p \in [0,1]$ , using Plurality is given by:
\begin{equation} \label{eq:t_p_sum}
\mathcal{T}(p) = \sum_{i=\lceil \frac{n}{2} \rceil}^{n}  \binom{n}{i} {p}^i (1-p)^{n-i}.
\end{equation}
\end{lemma}
\begin{proof}
It is possible to note how for $m=2$, the values of $\varphi_i$ is $\frac{1}{(n-i)!}$. 

This because:
\begin{align*}
\mathcal{G}^2_i(x) & = \left( \sum_{j=0}^{i-1} \frac{x^j}{j!} \right) \\
& = 1 + x + \frac{1}{2}x^2 + \cdots + \boldsymbol{\frac{1}{(n-i)!}}x^{n-i} + \cdots + \frac{1}{(i-1)!}x^{i-1}.
\end{align*}

Consequently,  with further algebraic simplifications, we have the following:

\begin{align*}
\mathcal{T}(p) 
&= \frac{1}{K}(1-p)^n \sum_{i=\lceil \frac{n}{2} \rceil}^{n} \boldsymbol{\varphi_i} (n-i)!  \binom{n}{i} \left( \frac{p}{1-p} \right) ^i \\
&= \frac{1}{K}(1-p)^n \sum_{i=\lceil \frac{n}{2} \rceil}^{n} \frac{\boldsymbol{(n-i)!}}{\boldsymbol{(n-i)!}}  \binom{n}{i} \left( \frac{p}{1-p} \right) ^i  \\
&= \frac{1}{\boldsymbol{K}}(1-p)^n \sum_{i=\lceil \frac{n}{2} \rceil}^{n}   \binom{n}{i} \left( \frac{p}{1-p} \right) ^i \\ 
&= \frac{(1-p)^n \sum_{i=\lceil \frac{n}{2} \rceil}^{n}  \binom{n}{i} \left( \frac{p}{1-p} \right) ^i}{\sum_{j=0}^{n} \binom{n}{j} p^j (1-p)^{\boldsymbol{n-j}}} \\ 
&= \frac{\boldsymbol{(1-p)^n} \sum_{i=\lceil \frac{n}{2} \rceil}^{n}  \binom{n}{i} \left( \frac{p}{1-p} \right) ^i}{\boldsymbol{(1-p)^n}\sum_{j=0}^{n} \binom{n}{j} \left(\frac{p}{1-p}\right)^j} \\ 
\Rightarrow \mathcal{T}(p) &= \frac{\sum_{i=\lceil \frac{n}{2} \rceil}^{n}  \binom{n}{i} \left( \frac{p}{1-p} \right) ^i}{\sum_{i=0}^{n} \binom{n}{i} \left( \frac{p}{1-p} \right)^i}.
\end{align*}

Now, looking at the denominator, by definition of binomial coefficient, we can note that:
$$
\sum_{i=0}^{n} \binom{n}{i} \left( \frac{p}{1-p} \right)^i = (1+\frac{p}{1-p})^n
= (1-p)^{-n}.$$
Thus, we obtain:
$$
\mathcal{T}(p) = \sum_{i=\lceil \frac{n}{2} \rceil}^{n}  \binom{n}{i} {p}^i (1-p)^{n-i}.
$$
\qed
\end{proof}

We will now consider the two cases separately: (i) $p=0.5$, and (ii) $p>0.5$ or $p<0.5$. For both cases we will prove the corresponding statement of Theorem~\ref{thm:limit}.

\subsubsection{Case: $p=0.5$}
We will now proceed to prove the second statement of Theorem~\ref{thm:limit}.

\begin{proof}
If $p=0.5$ we have that:

$$
\mathcal{T}(0.5) = \frac{\sum_{i=\lceil \frac{n}{2} \rceil}^{n}  \binom{n}{i}}{2^n }.
$$
We note that, if $n$ is an odd number:
$$ \sum_{i=\ceil*{\frac{n}{2}}}^{n}  \binom{n}{i} = \frac{\sum_{i=0}^{n}  \binom{n}{i}}{2} = 2^{n-1},$$
while if $n$ is even:
\begin{align*}
\sum_{i=\lceil \frac{n}{2} \rceil}^{n}  \binom{n}{i}  
&= \frac{\sum_{i=0}^{n}  \binom{n}{i}}{2} + \frac{1}{2}\binom{n}{\lceil \frac{n}{2} \rceil}. \\
\end{align*}
Thus, we have the two following cases, depending on $n$:
\begin{align}
\mathcal{T}(0.5) &=  \frac{2^{n-1}}{2^{n}} = 0.5, &\text{ if } n \text{ is odd}; \\
\mathcal{T}(0.5) &=  \frac{2^{n-1}+ \frac{1}{2}\binom{n}{\frac{n}{2}}}{2^{n}} = 0.5 + \frac{\frac{1}{2}\binom{n}{\frac{n}{2}}}{2^{n}}, &\text{ if } n \text{ is even}.
\end{align}
We can see that, when $n$ is odd, the following term becomes $0$ if $n$ tend to infinity:
\[
\lim_{n\to\infty} \frac{\frac{1}{2}\binom{n}{ \frac{n}{2}}}{2^{n}} = \lim_{n\to\infty} \frac{\binom{n}{ \frac{n}{2}}}{2^{n+1}} = 0.
\]
This limit is an indeterminate form $\frac{\infty}{\infty}$, that can be easily solved considering that $\binom{n}{ \frac{n}{2}} < {2^{n}}$. Given this observation we can see that the denominator prevails making the limit going to $0$.
Thus, we proved that:
\[
\lim_{n\to\infty} \mathcal{T}(0.5) = 0.5.
\]
\qed
\end{proof}

We note that if $n$ is odd $\mathcal{T}(0.5)=0.5$ also for small values of $n$, while if $n$ is even, $\mathcal{T}(0.5)$ converges to $0.5$ and it is equal to $0.5$ only when $n \to\infty$.

\subsection{Monotonicity and analysis of the derivative}
In this section, we first show that $\mathcal{T}(p)$ (see Equation~\ref{eq:t_p_sum}) is  monotonic increasing by proving that its derivative is greater or equal to zero.
Finally, we will see that, at the limit (for $n \to \infty$), the derivative is equal to zero for every $p \in [0,1]$ excluding $0.5$.

\begin{lemma}
The function $\mathcal{T}(p)$, describing the probability of electing the correct class $c^*$, among $2$ classes, with a profile of a $n$ classifiers, each one with accuracy $p \in [0,1]$ , using Plurality is monotonic increasing.
\end{lemma}
\begin{proof}
We know from Equation~\ref{eq:t_p_sum} in Lemma~\ref{lemma:t_p_sum} that 
$$\mathcal{T}(p) = \sum_{i=\lceil \frac{n}{2} \rceil}^{n}  \binom{n}{i} {p}^i (1-p)^{n-i} ~.$$ We want now to prove that $\mathcal{T}(p)\geq 0$.
\begin{align*}
\frac{\partial \mathcal{T}(p)}{\partial p}
& =   \frac{\partial \left(\sum_{i=\lceil \frac{n}{2} \rceil}^{n}  \binom{n}{i} {p}^i (1-p)^{n-i}\right)}{\partial p} \\
& =  \sum_{i=\lceil \frac{n}{2} \rceil}^{n}  \binom{n}{i}  \frac{\partial \left( {p}^i (1-p)^{n-i}\right)}{\partial p} \\
& =  \sum_{i=\lceil \frac{n}{2} \rceil}^{n}  \binom{n}{i} \left( \frac{\partial \left(     p^i \right)}{\partial p} (1-p)^{n-i} + p^i\frac{\partial \left(  (1-p)^{n-i}   \right)}{\partial p} \right)\\
& = \sum_{i=\lceil \frac{n}{2} \rceil}^{n}  \binom{n}{i} \left( i p^{i-1} (1-p)^{n-i} - p^i(n-i)(1-p)^{n-i-1}\right)\\
& =  \sum_{i=\lceil \frac{n}{2} \rceil}^{n}  \binom{n}{i} p^{i-1}(1-p)^{n-i-1}\left( i-pi-pn+pi\right)\\
& =  \sum_{i=\lceil \frac{n}{2} \rceil}^{n}  \binom{n}{i} p^{i-1}(1-p)^{n-i-1}\left( i-pn\right)\\
& =  (1-p)^{n-1} \sum_{i=\lceil \frac{n}{2} \rceil}^{n}  \binom{n}{i} p^{i-1}(1-p)^{-i}\left( i-p n \right)\\
& =  (1-p)^{n-1} \sum_{i=\lceil \frac{n}{2} \rceil}^{n}  \binom{n}{i}
\left(\frac{p}{1-p}\right)^{i-1} \left(\frac{i}{p}-n\right)\\
& =  (1-p)^{n-1} \left(\frac{1-p}{p} \right) \left\lceil \frac{n}{2} \right\rceil \binom{n}{\lceil \frac{n}{2} \rceil} \left( \frac{p}{1-p}\right)^{\lceil \frac{n}{2} \rceil}\\
& = p^{\lceil \frac{n}{2} \rceil -1} (1-p)^{n-\lceil \frac{n}{2} \rceil} \left\lceil \frac{n}{2} \right\rceil \binom{n}{\lceil \frac{n}{2} \rceil} \geq 0\\
\end{align*}
It is easy to see that the last row of the sequence is greater or equal to zero since each of the terms of the product is greater or equal to zero.
We proved that $\mathcal{T}(p)$ is monotonic increasing. \qed
\end{proof}

Let's see now that at the limit (with $n \to \infty$) the derivative is equal to zero for every $p\in[0,1]$ excluding $p=0.5$.

\begin{lemma}
Given the function $\mathcal{T}(p)$ describing the probability of electing the correct class $c^*$, among $2$ classes, with a profile of a $n$ classifiers, each one with accuracy $p \in [0,1]$ , using Plurality, we have that:
$$ \lim_{n\to\infty} \frac{\partial \mathcal{T}(p)}{\partial p} =0$$
\end{lemma}
\begin{proof}
Let's rewrite the function $\frac{\partial \mathcal{T}(p)}{\partial p}$ as follows:
$$
\frac{\partial \mathcal{T}(p)}{\partial p} 
= p^{\lceil \frac{n}{2} \rceil -1} (1-p)^{n-\lceil \frac{n}{2} \rceil} \left\lceil \frac{n}{2} \right\rceil \binom{n}{\lceil \frac{n}{2} \rceil} \\
$$
We will treat separately the case in which $n$ is an odd or even number:
$$ \frac{\partial \mathcal{T}(p)}{\partial p} =
\begin{cases*}
\left(p(1-p)\right)^{\lfloor \frac{n}{2} \rfloor }  \left\lceil \frac{n}{2} \right\rceil \binom{n}{\lceil \frac{n}{2} \rceil} & p odd \\ 
 \frac{(p(1-p))^{\frac{n}{2} }}{p}  \frac{n}{2} \binom{n}{ \frac{n}{2}} & p even \\ 
\end{cases*}
$$

{\bf Case 1: $n$ is odd.}
This is an indeterminate form $0\cdot\infty$, that can be solved considering that:
$$\left[p(1-p)\right]^{\lfloor \frac{n}{2} \rfloor }   \leq \frac{\partial \mathcal{T}(p)}{\partial p} \leq (2+\frac{1}{n})^n \left(p(1-p)\right)^{\lfloor \frac{n}{2} \rfloor }$$
where the inequality on the right follows from:
$$ 1< \left\lceil \frac{n}{2} \right\rceil \binom{n}{\lceil \frac{n}{2} \rceil} < (2+\frac{1}{n})^n.$$

Let's consider the function of the left inequality when $n\to\infty$.
\\Since $p(1-p)<1~\forall p \in [0,1]$, we know that:
$$ \lim_{n\to\infty} \left[p(1-p)\right]^{\lfloor \frac{n}{2} \rfloor } =0$$
 This can be proved with the following observation:
\begin{align*}
& p(1-p)<1 ~\forall p \in [0,1]
\Longleftrightarrow  (p- 1)^2  +p >0 ~\forall p \in [0,1]. \\
\end{align*}

Let's consider the function of the right inequality when $n\to\infty$:
 $$\lim_{n\to\infty} (2+\frac{1}{n})^n \left(p(1-p)\right)^{\lfloor \frac{n}{2} \rfloor } $$
 We know that this limit is zero because:
  $$(2+\frac{1}{n})^n \left(p(1-p)\right)^{\lfloor \frac{n}{2} \rfloor } = \left[(2+\frac{1}{n})^2 p(1-p)\right]^{\lfloor \frac{n}{2} \rfloor } $$
  and, given $p \in [0,1]$, always exists a value $N$ such that:
    $$\exists N>0 \text{ s.t. } \forall n>N, \left((2+\frac{1}{n})^2 p (1-p)\right) <1 \Longleftrightarrow p (1-p) < \frac{1}{4 + \frac{1}{n^2} + \frac{4}{n}} ~,$$
which for $n \to \infty$ holds if and only if $ p \neq \frac{1}{2}$.

We can now apply the \emph{squeeze theorem} and show that the derivative is equal to zero if $p \in [0,1], p \neq \frac{1}{2}$. It is important to notice that $\frac{\partial \mathcal{T}(p)}{\partial p}$ is not continuous in $p = \frac{1}{2}$.

{\bf  Case 2: $n$ is even.}
$$ \lim_{n\to\infty} \frac{\partial \mathcal{T}(p)}{\partial p} = \lim_{n\to\infty} \frac{1}{p}  ~~ (p(1-p))^{\frac{n}{2} } \frac{n}{2} \binom{n}{ \frac{n}{2}} $$ which is equivalent to:
$$ \frac{1}{p}  \lim_{n\to\infty}  (p(1-p))^{\frac{n}{2} } \frac{n}{2} \binom{n}{ \frac{n}{2}} $$

We saw before that:
$$ \lim_{n\to\infty} (p(1-p))^{\frac{n}{2} } \frac{n}{2} \binom{n}{ \frac{n}{2}} =0$$
Thus, the result holds also for the case in which $p$ is even.

\qed
\end{proof}
\subsubsection{Case: $p>0.5$ or $p<0.5$}

In the previous section, we proved that $\lim_{n\to\infty} \frac{\partial \mathcal{T}(p)}{\partial p}=0$ if $p\not=0.5$. This implies that we can rewrite $\mathcal{T}(p)$ for $n \to \infty$ in the following form:
\begin{equation} \label{eq:limit2}
\lim_{n\to\infty}\mathcal{T}(p) = 
\begin{cases}
v_1 & ~~~ p<0.5\\
v_2 & ~~~ p=0.5\\
v_3  & ~~~ p>0.5,
\end{cases}
\end{equation}
with $v_1$, $v_2$ and $v_3$ real numbers in $[0,1]$ such that $v_1 \leq v_2 \leq v_3$ (since $\mathcal{T}(p)$ is monotonic). We already proved that $v_2 = 0.5$. 

It is easy to see that $v_1 = 0$, because $\mathcal{T}(0) = 0, \forall n$ since all the terms of the sum are equal to zero. 
Finally, we have that $v_3 = 1$, because $\mathcal{T}(1) = 1, \forall n$.

In fact, $\mathcal{T}(1)$ corresponds to the probability of getting the correct prediction considering a profile of $n$ classifiers where each one elects the correct class with 100\% of accuracy. Since we are considering Plurality, which satisfies the axiomatic property of unanimity, the aggregated profile will also elect the correct class with 100\% of accuracy. Thus, the value of $\mathcal{T}(1)$ is $1$ for each $n>0$ and  consequently for $n\to\infty$. Thus, we showed that:
\begin{equation} \label{eq:limit3}
\lim_{n\to\infty}\mathcal{T}(p) = 
\begin{cases}
0 & ~~~ p<0.5\\
0.5 & ~~~ p=0.5\\
1  & ~~~ p>0.5,
\end{cases}
\end{equation}
This concludes the proof of Theorem~\ref{thm:limit}.

\section{Theoretical analysis: relaxing same-accuracy and independence assumptions}\label{sec:relaxing_assumptions}

In this section we will relax the assumptions made in Section~\ref{sec:tech-analysis} in two ways:
first, we remove the assumption that each classifier in the profile has the same accuracy $p$, allowing the classifiers to have a different accuracy (while still considering them independent); 
later we instead relax the independence assumption, allowing dependencies between classifiers by taking into account the presence of areas of the domain that are correctly classified by at least half of the classifiers simultaneously.

\subsection{Independent classifiers with different accuracy values}

Considering the same accuracy $p$ for all classifiers is not realistic, even if we set 
$p= \frac{1}{n}\sum_{i \in A} p_i$, that is, the average profile accuracy.
In what follows, we will relax this assumption by extending our study to the general case in which each classifier in the profile can have different accuracy, while still considering them independent. More precisely, we assume that each classifier $i$ has accuracy $p_i$ of choosing the correct class $c^*$.

In this case the probability of choosing the correct class for our ensemble method is:
$$
\frac{1}{K} \sum_{(S_1,\ldots,S_m) \in \Omega_{c^*}} \big[ \prod_{i \in \overline{S^*}} (1-p_i) \cdot \prod_{i \in S^*} p_i \big]
$$
where $K$ is the normalization function, $S$ is the set of all classifiers 
$S=\{1, 2, \ldots, n \}$; 
$S_i$ is the set of classifiers that elect candidate $c_i$; 
${S^*}$ is the set of classifiers that elect $c^*$; 
$\overline{S^*}$ is the complement of $S^*$ in S ($ \overline{S^*}=S \setminus S^*$);
and $\Omega_{c^*}$ is the set of all possible partitions of $S$ in which $c^*$ is chosen:
$$
    \Omega_{c^*} =  \{(S_1,\ldots,S_{m-1}) | \text{ partitions of } \overline{S^*} 
     \text{ s.t. } |S_i|< |S^*| ~ \forall i : c_i \not= c^*  \}.
$$

\cri{Notice that this scenario has been analyzed, although from a different point of view, in the literature (see for example \cite{nitzanparoush,shapleygrofman}). However, the focus of these works is fundamentally different from ours, since their goal is to find the optimal decision rule that maximizes the probability that a profile elects the correct class.}

\cri{Another relevant work is the one from \citet{epistemic_democracy} in which the authors study the case where  a profile of $n$ voters have to make a decision over $k$ options. Each voter $i$ has
independent probabilities $p_i^1, p_i^2, \cdots , p_i^k$ of voting for options $1, 2, \cdots , k$ respectively. The probability, $p_i^{c*}$ (i.e., the probability of voting for the correct outcome $c*$) exceeds each probabilities $p_i^{c}$ of voting for any of the incorrect outcomes, $c \not= c*$. The main difference with our approach is that in \citet{epistemic_democracy} the authors assume to know the full probability distribution over the outcomes for each voter, moreover they assume the voters have the same probability distribution. In this regard, we just assume to know the accuracy $p_i$ (different for each voter) for each classifier/voter (where $p_i = p_i^{c*}$). 
Thus, we provide a more general formula that covers more scenarios.}

\subsection{Dependent classifiers}

Until now, we assumed that the classifiers are independent: the set of the correctly classified examples of a specific classifier is selected by using an independent uniform distribution over all the examples. 

We now relax this assumption, by considering dependencies between classifiers by taking into account the presence of areas of the domain that are correctly classified by at least half of the classifiers simultaneously.
The idea is to estimate the amount of \textit{overlapping} of the classifications of the individual classifiers. 
We denote by $\varrho$ the ratio of the examples that are in the \textit{easy-to-classify} part of the domain (in which more than half of the classifiers is able to predict the correct label $c^*$). 
Thus, $\varrho$ equal to $1$ when the whole domain is \textit{easy-to-classify}.
Considering $n$ classifiers, we can define an upper-bound for $\varrho$:
$$  \varrho \leq \mathbb{P}[ \,\exists \, \mathcal{I}\, \subseteq S, \, |\mathcal{I}| \geq \frac{n}{2} \,\, \text{s.t.} \,\, \forall i \in \mathcal{I} \, \arg \max(x_i) = c^*]~.$$
In fact, $\varrho$ is bounded by the probability of the correct classification of an example by at least half of the classifiers (which are correctly classified by the ensemble). It is interesting to note that $\varrho \leq p$.
Removing the \textit{easy-to-classify} examples from the training dataset, we obtain the following accuracy for the other examples:
\begin{equation}
\label{eq:ptilde}
    \widetilde{p} = \frac{p -\varrho}{1 - \varrho} < p ~.
\end{equation}    
We are now ready to generalize Theorem \ref{prop:1}.

\begin{theorem}
\label{th:overlap}
The probability of choosing the correct class $c^*$ in a profile of $n$ classifiers with accuracy $p \in [0,1[$, $m$ classes and with an overlapping value $\varrho$, using Plurality to compute the winner, is larger than:
\begin{equation}
\label{eq:final_mik}
(1 -\varrho)\mathcal{T}(\widetilde{p})  + \varrho 
~.
\end{equation}
\end{theorem}
\noindent The statement follows from Theorem \ref{prop:1} and splitting the correctly classified examples by the ratio defined by $\varrho$.
This result tells us that, in order to obtain an improvement of the individual classifiers' accuracy $p$, we need to maximize the Formula~\ref{eq:final_mik}.
This corresponds to avoid maximizing the overlap $\varrho$ 
(the ratio of the examples that are in the \textit{easy-to-classify} in which more than half of the classifiers is able to predict the correct label) 
since this would lead to a counter-intuitive effect: if we maximize the overlap of a set of classifiers with accuracy $p$, in the optimal case the accuracy of the ensemble would be $p$ as well (we recall that $\varrho$ is bounded by $p$). Our goal is instead to obtain a collective accuracy greater than $p$. Thus, the idea is that we want to focus also on the examples that are more difficult to classify.

The ideal case, to improve the final performance of the ensemble, is to generate a family of classifiers with a balanced trade-off between $\varrho$ and the portion of accuracy generated by classifying the difficult examples (i.e., the ones not in the \emph{easy-to-classify} set). A reasonable way to pursue this goal corresponds to choosing the base classifiers randomly.

\begin{example}
Consider $n=10$ classifiers with $m=2$ classes and assume the accuracy of each classifier in the profile is $p=0.7$. Following the previous observations, we know that $\varrho \leq 0.7$. In the case of the maximum overlap among classifiers, i.e., $\varrho = 0.7$, the accuracy of VORACE is $0.3 \mathcal{T}(\widetilde{p})  + 0.7$. Recalling Eq. \ref{eq:ptilde}, we have that $\widetilde{p} = 0$ and, consequently, $\mathcal{T}(\widetilde{p}) = \mathcal{T}(0) = 0$. Thus, the accuracy of VORACE remains exactly $0.7$. 
In general (see Figure \ref{fig:formula_p_n}), with small values for the input accuracy $p$, the function $\mathcal{T}(p)$ obtains a decrease of the original accuracy. On the other hand, in the case of a smaller overlap, for example the edge case of $\varrho = 0$, we have that $\widetilde{p}=p$, and Formula \ref{eq:final_mik} becomes equal to the original Formula \ref{eq:prob_exact}. Then, VORACE is able to exploit the increase of performance given by $n=10$ classifiers with a high $\widetilde{p}$ of $0.7$. In fact, Formula \ref{eq:final_mik} becomes simply $\mathcal{T}(0.7)$ that is close to $0.85 > 0.7$, improving the accuracy of the final model.
\end{example}

\section{Conclusions and Future Work}
\label{sec-con}


We have proposed the use of voting rules in the context of ensemble classifiers:
a voting rule aggregates the predictions of several randomly generated classifiers, with the goal to obtain a classification that is closer to the correct one.
Via a theoretical and experimental analysis, we have shown that this approach generates ensemble classifiers that perform similarly to, or even better than,
existing ensemble methods.
This is especially true when VORACE employs Plurality or Copeland as voting rules. In particular, Plurality has also the added advantage to 
require very little information from the individual classifiers and being tractable.
Compared to building ad-hoc classifiers that optimize the hyper-parameters configuration for a specific dataset, our approach does not require any knowledge of the domain and thus it is more broadly usable also by non-experts.

We plan to extend our work to deal with other types of data, such as structured data, text, or images. This will also allow for a direct comparison of our approach with the work by \cite{BergstraB12}.
Moreover, we are working on extending the theoretical analysis beyond the Plurality case. 

We also plan to consider the extension of our approach to multi-class 
classification. In this regard, a prominent application of voting theory to this scenario might come from the use of committee selection voting rules \cite{FSST17a} in an ensemble classifier. 
properties of voting rules that may be relevant and desired in the classification domain (see for instance \cite{Grandi2014FromSA,grandi2016borda}), with the aim to identify and select voting rules that possess such properties, or to define new voting rules with these properties, or also to prove impossibility results about the presence of one or more such properties.
We also plan to study 


%
%

\bibliographystyle{spbasic}      
\bibliography{bibliografia}   

%
%

\newpage
\appendix
\section{Discussion and comparison with \cite{Mu2009}.}\label{appendix:Mu2009}
In this section, we compare our theoretical formula to estimate the accuracy of VORACE in Eq. \ref{eq:prob_exact} (for the plurality case) with respect to the one provided in \citet{Mu2009} (page 93 Section 3.2, formula for $P_{id}$ Eq. 8), providing details of the problem of their formulation. From our analysis, we discovered that applying their estimation of the -- so called -- Identification Rate ($P_{id}$) produces incorrect results, even in simple cases. We can prove it by using the following counterexample: a binary classification problem where the goal is ``to combine'' a single classifier with accuracy $p$, i.e., number of classes $m=2$, and number of classifiers $n=1$. It is straightforward that the final accuracy of a combination of a single classifier with accuracy $p$ has to remain unchanged ($P_{id} = p$).

Before proceeding with the calculations, we have to introduce some quantities, following the same ones defined in their original paper:
\begin{itemize}
    \item $N_t$ is a random variable that gives the total number of votes received by the correct class:
    \[
        P(N_t=j) = \binom{n}{j} p^j (1-p)^{n-j}.
    \]
    \item $N_s$ is a random variable that gives the total number of votes received by the wrong class $s^{th}$:
    \[
        P(N_s=j) = \binom{n}{j} e^j (1-e)^{n-j},
    \]
    where $e = \frac{1 - p}{m - 1}$ is the misclassification rate.
    \item $N_s^{max}$ is a random variable that gives the maximum number of votes among all the wrong classes:
    \begin{align*}
        &P(N_s^{max} = k) = \\
        = &\sum_{h=1}^{m-1} \binom{m-1}{h} P(N_s = k)^h P(N_s < j)^{m-1-h},
    \end{align*}
    where the quantity $P(N_s < j)$ is:
    \[
        P(N_s < j) = \sum_{t=0}^{j-1} P(N_s=t).
    \]
\end{itemize}
The authors assume that $N_t$ and $N_s^{max}$ are independent random variables. 
This means that the probability that the correct class obtains $k$ votes is independent to the probability that the maximum votes within the wrong classes correspond to $j$. 
This false assumption leads to a wrong final formula. 
In fact, applying Eq. 8 in \cite{Mu2009} to our simple binary scenario with a single classifier, we have that the new estimated accuracy is:
\begin{align}
    P_{id} &= \sum_{j=1}^N P(N_t=j) \sum_{k=0}^{j-1} P(N_s^{max} = k) = \\
    &= P(N_t = 1) P(N_s^{max} = 0) = p^2, \nonumber
\end{align}
whereas the correct result should be $p$.

On the other hand, our proposed formula (Theorem \ref{prop:1}) tackles this scenario correctly, as proved in the following, where we specify Equation \ref{eq:prob_exact} to this context:
\begin{align*}
P_{id} &= \frac{1}{K}(1-p)^n \sum_{i=\lceil \frac{n}{m} \rceil}^{n} \varphi_i (n-i)!  \binom{n}{i} \left( \frac{p}{1-p} \right) ^i\\
&= \frac{1}{K} \varphi_1(0)! p =  p,
\end{align*}
where $\varphi_1(0)! = 1$ and $K = 1$.

Notice that, as expected, Formula \ref{eq:prob_exact} is equal to $1$ when $p=1$, meaning that, when all classifiers are correct, our ensemble method correctly outputs the same class as all individual classifiers.

As other proof of the difference between the two formulas, we created a similar plot as the one in Figure \ref{fig:formula_p_n}, applying  Eq. 8 in \cite{Mu2009} -- instead of our formula -- obtaining Figure \ref{fig:formula_p_n_nope}. The two plots are similar, with a  less steepness in the curves generated by using our formula. In this sense, we suppose that the formula proposed by \cite{Mu2009} is a good approximation of the correct value of $P_{id}$ for large values of $n$ (as we proved that for $n=1$ and $m=2$ is not correct).

\begin{figure}
\centering
\begin{tikzpicture}
\normalsize
\begin{axis}[legend style={at={(0.8,0.35)},
		anchor=north,legend columns=1}, xmin=0,xmax=1,ymin=0, ymax=1, xlabel=p (accuracy of the base classifiers), ylabel=Probability of choosing the correct class,xmajorgrids, ymajorgrids,every axis plot/.append style={thick}]
\addplot+[gray,mark=,solid,smooth] 
coordinates
{
(0,0)
(0.05,0.05)
(0.1,0.1)
(0.15,0.15)
(0.2,0.2)
(0.25,0.25)
(0.3,0.3)
(0.35,0.35)
(0.4,0.4)
(0.45,0.45)
(0.5,0.5)
(0.55,0.55)
(0.6,0.6)
(0.65,0.65)
(0.7,0.7)
(0.75,0.75)
(0.8,0.8)
(0.85,0.85)
(0.9,0.9)
(0.95,0.95)
(1,1)
};
\addplot+[red,mark=, solid,smooth] 
coordinates
{
(0.0,0.0)
(0.05,5.37960058398467e-10)
(0.1,7.088606331722193e-07)
(0.15000000000000002,3.86327482081078e-05)
(0.2,0.0005634136976601906)
(0.25,0.003942141664083465)
(0.30000000000000004,0.017144816431258456)
(0.35000000000000003,0.05316661436294621)
(0.4,0.12752124614721674)
(0.45,0.24928935982841194)
(0.5,0.41190147399902344)
(0.55,0.5913611846716277)
(0.6000000000000001,0.7553372033163934)
(0.65,0.878219413622599)
(0.7000000000000001,0.9520381026686567)
(0.75,0.9861355830562388)
(0.8,0.997405172599326)
(0.8500000000000001,0.9997516180103759)
(0.9,0.9999928490959789)
(0.9500000000000001,0.9999999886592819)
(1.0,1.0)
};
\addplot+[green,mark=,solid,smooth] 
coordinates
{
(0.0,0.0)
(0.05,3.7430554296596106e-39)
(0.1,6.323256842131429e-25)
(0.15000000000000002,3.942440434927658e-17)
(0.2,5.179892637524884e-12)
(0.25,2.131191608390269e-08)
(0.30000000000000004,9.03468619572068e-06)
(0.35000000000000003,0.0007378332488626113)
(0.4,0.016761686503161403)
(0.45,0.13457621318805263)
(0.5,0.46020538130641064)
(0.55,0.8172718153138552)
(0.6000000000000001,0.9729008022429914)
(0.65,0.9985494385234414)
(0.7000000000000001,0.9999779390866731)
(0.75,0.9999999336149751)
(0.8,0.9999999999786072)
(0.8500000000000001,0.9999999999999999)
(0.9,1.0)
(0.9500000000000001,0.9999999999999999)
(1.0,1.0)
};
\addplot+[blue,mark=,solid,smooth] 
coordinates
{
(0.0,0.0)
(0.05,2.3222802751611695e-75)
(0.1,2.964852156817343e-47)
(0.15000000000000002,6.783947104813922e-32)
(0.2,7.626779260241388e-22)
(0.25,8.778706966466265e-15)
(0.30000000000000004,1.0864028884363591e-09)
(0.35000000000000003,4.940555931357384e-06)
(0.4,0.0016847865199193523)
(0.45,0.06807524986274897)
(0.5,0.4718257604953717)
(0.55,0.9112993844383077)
(0.6000000000000001,0.9973645966438089)
(0.65,0.9999905406616545)
(0.7000000000000001,0.9999999974041742)
(0.75,0.9999999999999729)
(0.8,0.9999999999999997)
(0.8500000000000001,0.9999999999999998)
(0.9,1.0)
(0.9500000000000001,0.9999999999999999)
(1.0,1.0)
};
\legend{$p$, $n=10$, $n=100$, $n=400$}
\end{axis}
\end{tikzpicture}
\caption{Probability of choosing the correct class ($P_{id}$) varying the size of the profile $n$ in $\{10,50,100\}$ and keeping $m$ constant to 2, where each classifier has the same probability $p$ of classifying a given instance correctly, by using Eq. 8 in \cite{Mu2009}.}
\label{fig:formula_p_n_nope}
\end{figure}
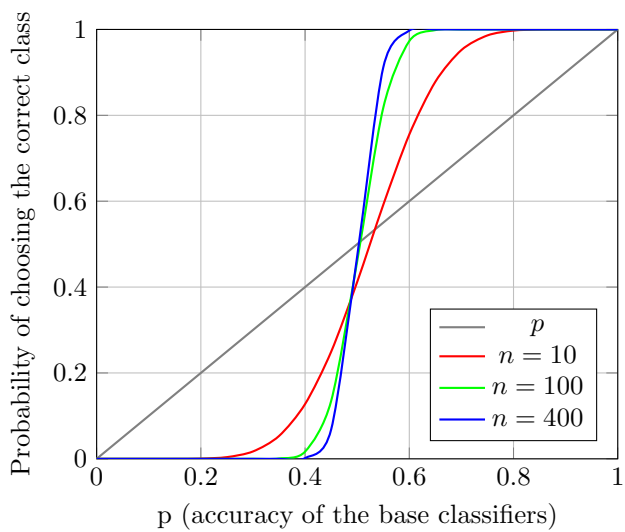

\end{document}